\documentclass{article}

\usepackage{microtype}
\usepackage{graphicx}
\usepackage{subfigure}
\usepackage{booktabs} 

\usepackage{hyperref}

\usepackage[accepted]{arxiv_version}

\icmltitlerunning{On the Connection Between Adversarial Robustness and Saliency Map Interpretability}

\usepackage{amsmath}
\usepackage{amsthm}
\usepackage{amssymb}
\usepackage{thmtools} 
\usepackage{thm-restate}

\newtheorem{definition}{Definition}
\newtheorem{lemma}{Lemma}

\usepackage{verbatim}
\usepackage{tikz}
\usepackage{pgfplots}
\DeclareMathOperator*{\D}{\mathcal{D}}

\DeclareMathOperator*{\loss}{\mathcal{L}}
\def\R{{\mathbb{R}}}

\DeclareMathOperator*{\argmax}{arg\,max}

\DeclareMathOperator*{\sgn}{sgn}

\def\E{{\mathbb{E}}}

\newcommand{\vertiii}[1]{{\left\vert\kern-0.25ex\left\vert\kern-0.25ex\left\vert #1 
    \right\vert\kern-0.25ex\right\vert\kern-0.25ex\right\vert}}

\newcommand{\Expectx}[1]{\E_{x \sim \D} \left[ #1 \right]}

\newcommand{\alignment}{\alpha^{\dagger}}

\begin{document}

\twocolumn[
\icmltitle{On the Connection Between Adversarial Robustness and Saliency Map Interpretability}



\icmlsetsymbol{equal}{*}

\begin{icmlauthorlist}
\icmlauthor{Christian Etmann}{equal,hb,camv}
\icmlauthor{Sebastian Lunz}{equal,cam}
\icmlauthor{Peter Maass}{hb}
\icmlauthor{Carola-Bibiane Sch\"{o}nlieb}{cam}
\end{icmlauthorlist}

\icmlaffiliation{hb}{Center for Industrial Mathematics, University of Bremen, Bremen, Germany}
\icmlaffiliation{cam}{DAMTP, University of Cambridge, Cambridge, United Kingdom}
\icmlaffiliation{camv}{Work done at DAMTP, Cambridge.}

\icmlcorrespondingauthor{Christian Etmann}{cetmann@math.uni-bremen.de}
\icmlcorrespondingauthor{Sebastian Lunz}{lunz@math.cam.ac.uk}

\icmlkeywords{Machine Learning, ICML}

\vskip 0.3in
]



\printAffiliationsAndNotice{\icmlEqualContribution} 

\begin{abstract}
Recent studies on the adversarial vulnerability of neural networks have shown that models trained to be more robust to adversarial attacks exhibit more interpretable saliency maps than their non-robust counterparts. We aim to quantify this behavior by considering the alignment between input image and saliency map. We hypothesize that as the distance to the decision boundary grows, so does the alignment. This connection is strictly true in the case of linear models. We confirm these theoretical findings with experiments based on models trained with a local Lipschitz regularization and identify where the non-linear nature of neural networks weakens the relation.
\end{abstract}
\section{Introduction}
Despite impressive results in a variety of classification tasks \citep{lecun2015deep}, even highly accurate neural network classifiers are plagued by a vulnerability to so-called \emph{adversarial perturbations} \citep{Szegedy2013IntriguingPO}. These adversarial perturbations are small, often visually imperceptible perturbations to the network's input, which however result in the network's classification decision being changed. Such vulnerabilities may pose a threat to real-world deployments of automated recognition systems, especially in security-critical applications such as autonomous driving or banking. This has sparked a large number of publications related to both the creation of adversarial attacks \citep{explainingharnessing,kurakin2016adversarial,moosavi2016deepfool} as well as defenses against these (see \citep{schott} for an overview). Apart from the application-focused viewpoint, the observed adversarial vulnerability offers non-obvious insights into the inner workings of neural networks.
One particular method of defense is \emph{adversarial training} \citep{madry2017towards}, which aims to minimize a modified training objective. While this method -- like all known approaches of defense -- decreases the accuracy of the classifier, it is also successful in increasing the robustness to adversarial attacks, i.e. the perturbations need to be larger on average in order to change the classification decision.
\begin{figure}
    \centering
    \includegraphics[width=.95\linewidth]{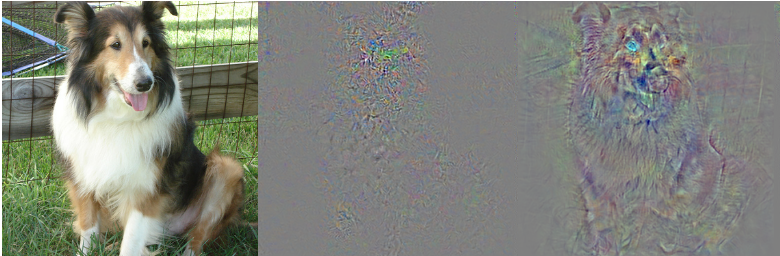}
    \caption{An image of a dog (left), the saliency maps of a highly non-adversarially-robust neural network (middle) and of a more robust network (right). We observe that the robust network gives a much clearer indication of what the classifier deems to be discriminative features. Details about saliency and the robustification are given in section \ref{sec:experiments}. Most figures are best viewed on a screen.}
    \label{fig:dog}
\end{figure}

\citep{tsipras2018robustness} also notice that networks that are robustified in this way show interesting phenomena, which so far could not be explained. Neural networks usually exhibit very unstructured \emph{saliency maps} (gradients of a classifier score with respect to the network's input \citep{simonyan2013deep}) which barely relate to the input image. On the other hand, saliency maps of robustified classifiers tend to be far more interpretable, in that structures in the input image also emerge in the corresponding saliency map, as exemplified in Figure \ref{fig:dog}. \citep{tsipras2018robustness} describe this as an 'unexpected benefit' of adversarial robustness. In order to obtain a semantically meaningful visualization of the network's classification decision in non-robustified networks, the saliency map has to be aggregated over many different points in the vicinity of the input image. This can be achieved either via averaging saliency maps of noisy versions of the image \citep{smilkov2017smoothgrad} or by integrating along a path \citep{sundararajan2017axiomatic}. Other approaches typically employ modified backpropagation schemes in order to highlight the discriminative portions of the image. Examples of this include \emph{guided backpropagation} \citep{DB15a} and \emph{deep Taylor decomposition} \citep{montavon2017explaining}.\newline
In this paper, we show that the interpretability of the saliency maps of a robustified neural network is not only a side-effect of adversarial training, but a general property enjoyed by networks with a high degree of robustness to adversarial perturbations. We first demonstrate this principle for the case of a linear, binary classifier and show that the 'interpretability' is due to the image vector and the respective image gradient aligning. For the more general, non-linear case we empirically show that while this relationship is true on average, the linear theory and the non-linear reality do not always agree. We empirically demonstrate that the more linear the model is, the stronger the connection between robustness and alignment becomes.

\section{Adversarial Robustness and Saliency Maps}
Since adversarial perturbations are small perturbations that change the predicted class of a neural network, it makes sense to define the robustness towards adversarial perturbations via the distance of the unperturbed image to its nearest perturbed image, such that the classification is changed.
\begin{definition}
Let $F: X \rightarrow C$ (with $C$ finite) be a classifier over the normed vector space $(X,\| \cdot \|)$. We call 
\begin{align}
\label{eq:generic_robustness}
    \rho(x)=\inf_{e \in X} \lbrace \|e\| : F(x+e) \neq F(x) \rbrace
\end{align}
the \emph{(adversarial) robustness of $F$ in the point $x$}. We call $\Expectx{\rho(x)}$ the \emph{(adversarial) robustness of $F$ over the distribution $\mathcal{D}$}.
\end{definition}
Put differently, the robustness of a classifier in a point is nothing but the distance to its closest decision boundary. Margin classifiers like support vector machines \citep{cortes1995support} seek to keep this distance large for the training set, usually in order to avoid overfitting. \citep{sokolic2017robust} and \citep{47365} also apply this principle to neural networks via regularization schemes. We point out that our definition of adversarial robustness does not depend on the ground truth class label and -- given feasible computability -- can approximately be calculated even on unlabelled data.\\
In the following, we will always assume $X$ to be a real, finite-dimensional vector space with the Euclidean norm. The proofs for the following theoretical statements are found in the appendix.

\subsection{A Motivating Toy Example}\label{sec:toy_examples}
We consider the toy case of a linear binary classifier $F(x) = \sgn(\Psi_z(x))$ with the so-called score function $\Psi_z(x)=\langle x,z \rangle$ and fixed \mbox{$z\neq 0$}, where $\langle \cdot , \cdot \rangle$ denotes the standard inner product on $\R^m$. A straightforward calculation (see appendix) shows that the adversarial robustness of $F$ is given by
\begin{equation} \label{eq:linear_robustness}
    \rho(x) = \frac{|\langle x,z \rangle |}{\|z\|} = \frac{|\langle x,\nabla \Psi_z(x) \rangle |}{\|\nabla \Psi_z(x)\|}.
\end{equation}
Unless stated otherwise, we will always denote with $\nabla$ the gradient with respect to $x$. Note that \mbox{$\rho(x) = \|x\| \cdot |\cos(\delta) |$}, where $\delta$ is the angle between the vectors $x$ and $\nabla \Psi_z(x)$. This implies that $\rho(x)$ grows with the alignment of $x$ and $z$ and is maximized if and only if $x$ and $z$ are collinear.\newline
This motivates the following definition. 
\begin{definition}[Alignment]
Let the binary classifier $$F: X \rightarrow \{-1,1\}$$ be defined a.e. by $F(x)=\text{\emph{sgn}}(\Psi(x))$, where
$\Psi: X \rightarrow \R$ is differentiable in $x$. We then call $\nabla \Psi$ the \emph{saliency map of $F$ with respect to $\Psi$ in $x$} and \begin{equation} \label{eq:alignment_definition}
    \alpha(x) := \frac{|\langle x,\nabla \Psi(x) \rangle |}{\|\nabla \Psi(x)\|},
\end{equation} the alignment with respect to $\Psi$ in $x$.
\end{definition}
The alignment is a measure of how similar the input image $x$ and the saliency map $\nabla \Psi(x)$ are. If $\|x\|=1$, and $x$ and $\nabla\Psi(x)$ are zero-centered, this coincides with the absolute value of their Pearson correlation. For a linear binary classifier, the alignment trivially increases with the robustness of the classifier.

Generalizing from the linear to the affine case leads to a classifier of the form $F(x) = \sgn (\langle x,z \rangle + b)$, whose robustness in $x$ is $$\rho(x)=\frac{|\langle x,z \rangle + b |}{\|z\|}.$$ In this case the robustness and alignment do not coincide anymore. In order to connect these two diverging concepts, we offer two alternative viewpoints. On the one hand, we can trivially bound the robustness via the triangle inequality
\begin{align}\label{equ:view_one}
    \rho(x) \leq \alpha(x) + \frac{|b|}{\|z\|}.
\end{align}
This is particularly meaningful if $|b|/\|z\|$ is small in comparison to $\alpha(x)$. Alternatively, one can connect the robustness to the alignment at a different point $\xi=x+\frac{b}{\|z\|}\frac{z}{\|z\|}$, leading to the relation
\begin{align}\label{equ:view_two}
    \rho(x) = \alpha(\xi).
\end{align}
In the affine case this approach simply amounts to a shift of the data that is uniform over all data points $x$. We will see how these two viewpoints lead to different bounds in the non-linear case later. 

\subsection{The General Case}\label{sec:general_case}
We now consider the general, $n$-class case.
\begin{definition}[Alignment, Multi-Class Case]
Let $$\Psi=(\Psi^1,\dots,\Psi^n): X \rightarrow \R^n$$ be differentiable in $x$. Then for an $n$-class classifier defined a.e. by \begin{align}
\label{equ:Multi_score}
    F(x) = \argmax_i \Psi^i(x),
\end{align} we call $\nabla \Psi^{F(x)}$ the \emph{saliency map of $F$}. We further call \begin{equation} \label{eq:multi_alignment_definition}
    \alpha(x):=\frac{|\langle x, \nabla \Psi^{F(x)}(x) \rangle|}{\|\nabla \Psi^{F(x)}(x)\|},
\end{equation} the \emph{alignment with respect to $\Psi$ in $x$}.
\end{definition}

\subsubsection{Linearized Robustness}
In general the distance to the decision boundary $\rho(x)$ can be unfeasible to compute. However, for classifiers built on locally affine score functions -- such as most neural networks using ReLU or leaky ReLU activations -- $\rho(x)$ can easily be computed, provided the locally affine region is sufficiently large. To quantify this, define the radius of the locally affine component of $F$ around $x$ as
\begin{align*}
    l(x) = \sup \{ r \ | \ \forall i: \Psi^i \text{ affine in } B_r(x) \},
\end{align*}
where $B_r(x)$ is the open ball of radius $r$ around $x$ with respect to the Euclidean metric.
\begin{restatable}{lemma}{robLemma}
    Let $F$ be a classifier with locally affine score function $\Psi$. Assume $l(x) \geq \rho(x)$. Then
    \begin{align}
        \rho(x) = \min_{j \neq {i^\ast}} \frac{\Psi^{i^\ast}(x) - \Psi^j(x)}{\|\nabla \Psi^{i^\ast}(x) - \nabla \Psi^j(x)\|},
        \label{eq:localaffinerho}
    \end{align}
    for ${i^\ast}:=F(x)$ the predicted class at $x$.
\end{restatable}
Similar identities were previously also independently derived in \citep{47365} and \citep{Jakubovitz_2018_ECCV}.

Note that while nearly all state-of-the art classification networks are piecewise affine, the condition $l(x) \geq \rho(x)$ is typically violated in practice. However, the lemma can still hold \textit{approximately} as long as the linear approximation to the network's score functions is sufficiently good in the relevant neighbourhood of $x$. This motivates the definition of the \emph{linearized (adversarial) robustness} $\tilde{\rho}$.
\begin{definition}[Linearized Robustness] Let $\Psi(x)$ be the  differentiable score vector for the classifier $F$ in $x$. We call
  \begin{align}\label{equ:linearized_robustness}
      \tilde{\rho}(x):= \min_{j \neq {i^\ast}} \frac{\Psi^{i^\ast}(x) - \Psi^j(x)}{\|\nabla \Psi^{i^\ast}(x) - \nabla \Psi^j(x)\|},
  \end{align}
  the \emph{linearized robustness} in $x$, where ${i^\ast}:=F(x)$ is the predicted class at point $x$.
\end{definition}

We later show that the two notions lead to very similar results, even if the condition $l(x) \geq \rho(x)$ is violated.

\subsubsection{Reducing the Multi-Class Case}
In this section, we introduce a toolset which helps bridge the gap between the alignment and the linearized robustness of a multi-class classifier. In the following, for fixed $x$, let $i^\ast:=F(x)$ and $j^\ast$ be the minimizer in \eqref{equ:linearized_robustness}. We can assign $F$ in $x$ a \emph{binarized classifier} $F^\dagger_x$ with
\begin{equation}
    F^\dagger_x(y) := \sgn \left( \Psi^\dagger_x(y) \right),
\end{equation} 
where $\Psi^\dagger_x(y):=\Psi^{i^\ast}(y) - \Psi^{j^\ast}(y)$. Its linearized robustness in $y=x$ is the same as for $F$. The \emph{binarized saliency map}, $\nabla \Psi^\dagger_x(x) = \nabla_y \Psi^\dagger_x(y)|_{y=x}$ and the respective alignment,
\begin{equation}
    \alignment(x) = \frac{|\langle x, \nabla(\Psi^{i^\ast} - \Psi^{j^\ast})(x) \rangle|}{\|\nabla(\Psi^{i^\ast} - \Psi^{j^\ast})(x)\|},
\end{equation}
which we call \emph{binarized alignment}, offer an alternative, natural perspective of the above considerations. This is because for classifiers as defined in \eqref{equ:Multi_score}, the actual score values do not necessarily carry any information about the classification decision, whereas the score differences do. While, roughly speaking, $\nabla \Psi^{i^\ast}$ tells us what $F$ 'thinks' makes $x$ a member of its predicted class, $\nabla \Psi^\dagger_x(x)$ carries information what sets $x$ apart from its closest neighboring class (according to linearization).\newline

In the special case of a linear, multi-class classifier, we have $$\rho(x)=\tilde{\rho}(x)=\alignment(x)$$
and in the linear, binary case \mbox{$\Psi(x)=(\langle x,z \rangle , -\langle x,z \rangle)$}, even $$\alpha(x) = \alignment(x).$$

\section{Decompositions and Bounds for Neural Networks}

\subsection{Homogeneous Decomposition}
In the previous chapter we have seen that in the case of binary classifiers, the robustness and binarized alignment coincide for linear score functions. However, requiring $\Psi$ to be linear is a stronger assumption than necessary to deduce the result: It is in fact sufficient for $\Psi$ to be \textit{positive one-homogeneous}. Any such function satisfies $\Psi(a x) = a \Psi(x)$ for all $a>0$ and $x$. 

\begin{restatable}[Linearized Robustness of Homogeneous Classifiers]{lemma}{homoAlign}
    Consider a classifier $F$ with positive one-homogeneous score functions. Then
    \begin{align}
        \tilde{\rho}(x) = \alignment(x).
    \end{align}
\end{restatable}

In particular, most feedforward neural networks with (leaky) ReLU activations \emph{without biases} are positive one-homogeneous. This observation motivates to split up any classifier built on neural networks into a homogeneous term and the corresponding remainder, leading to the following decomposition result.

\begin{restatable}[Homogeneous Decomposition of Neural Networks]{theorem}{homoDec}
Let $\Psi_{\Theta, b}^{i}$ be any logit of a neural network with ReLU activations (of class $\mathcal{N}$ in the appendix). Denote by $\Theta$ the linear filters and by $b$ the bias terms of the network. Then 
\begin{align}
\begin{split}
    \Psi_{\Theta, b}^i(x) \ &= \langle x, \nabla_x \Psi_{\Theta, b}^i(x) \rangle + \langle b, \nabla_b \Psi_{\Theta, b}^i(x) \rangle
    \\
    &= \langle x, \nabla_x \Psi_{\Theta, b}^i(x) \rangle + \sum_k b_k \partial_{b_k} \Psi_{\Theta, b}^i(x).
\end{split}
\end{align}
\end{restatable}

Note that the above vector $b$ includes the running averages of the means for batch normalization. For ReLU networks, the remainder term $\beta^i(x):=\langle b, \nabla_b \Psi_{\Theta, b}^i(x) \rangle$ is locally constant, because it changes only when $x$ enters another locally linear region. For ease of notation, we will now drop the subscripts $\Theta$ and $b$.
\subsection{Pointwise Bounds}\label{sec:bounds}
In section \ref{sec:toy_examples}, we introduced two different viewpoints for affine linear, binary classifiers which connect the robustness to the alignment. In a similar vein to inequality \eqref{equ:view_one} and equality \eqref{equ:view_two}, upper bounds to the linearized robustness depending on the alignment can be given for neural networks. In the following, we will write $\overline{v}:=v/\|v\|$ for $v\neq 0$. Again, in the following we fix $x$ and write $i^\ast=F(x)$ and $j^\ast$ for the minimizer in $j$ from equation \eqref{equ:linearized_robustness}.

\begin{restatable}[]{theorem}{boundOne}
\label{thm:x_bounds} \label{thm:view_one}
Let $g:= \nabla \Psi^{i^\ast} (x)$. Furthermore, let  \mbox{$g^\dagger:= \nabla (\Psi^{i^\ast} - \Psi^{j^\ast})(x)$} and $\beta^\dagger:=\beta^{i^\ast}(x) - \beta^{j^\ast}(x)$. Then
\begin{align}
    \tilde{\rho}(x) &\leq \alignment(x) + \frac{|\beta^{\dagger}|}{\|g^{\dagger}\|} \label{eq:theorem2a}\\
    &\leq \alpha(x) + \| x\| \cdot \| \overline{g}^{\dagger}-\overline{g}\| + \frac{|\beta^{\dagger}|}{\|g^{\dagger}\|}.\label{eq:theorem2b}
\end{align}
\end{restatable}

Distances on the unit sphere (such as $\| \overline{g}^{\dagger}-\overline{g}\|$) can be converted to angles through the law of cosines. For the above inequalities to be reasonably tight, the angle between $g$ and $g^\dagger$ needs to be small and $|\beta^\dagger|/\|g^\dagger\|$ needs to be small in comparison to $\alpha^\dagger(x)$. In this case, the alignment should roughly increase with the linearized robustness.

\begin{restatable}[]{theorem}{boundTwo}
\label{thm:xi_bounds} Let $\xi:= x + \tfrac{\beta^\dagger}{\|g^\dagger\|}\tfrac{g^\dagger}{\|g^\dagger\|}$ and $\gamma:=\nabla \Psi^{i^\ast}(\xi)$, with $g^\dagger$ and $\beta^\dagger$ defined as in the previous theorem. Then
\begin{equation}
    \begin{aligned}\label{eq:theorem3a}
        \tilde{\rho}(x) &\leq \frac{|\langle \xi, \gamma \rangle|}{ \|\gamma\|} + \| \xi \| \cdot \|\overline{g}^{\dagger} - \overline{\gamma} \|,
    \end{aligned}
\end{equation}
and if additionally $F(x)=F(\xi),$ then
\begin{equation*}
    \begin{aligned}
        \tilde{\rho}(x) & \leq\alpha(\xi) + \| \xi \| \cdot \|\overline{g}^{\dagger} - \overline{\gamma} \|.\label{eq:theorem3b}
    \end{aligned}
\end{equation*}
\end{restatable}

Depending on the sign of $\beta^\dagger$, the shifted image $\xi$ can either be understood as a gradient ascent or descent iterate for maximizing/minimizing $\Psi^{i^\ast} - \Psi^{j^\ast}$. This theorem assimilates $\beta^\dagger(x)$ into $x$, providing an upper bound to $\tilde{\rho}(x)$ that depends on $\alpha(\xi)$. The sensibility of this hinges on $\xi$ being reasonably close to $x$ and $\gamma$ having a low angle with $g^\dagger$. 

If the error terms in inequalities \eqref{eq:theorem2a}, \eqref{eq:theorem2b} and \eqref{eq:theorem3a} are small, these inequalities thus provide a simple illustration why more robust networks yield more interpretable saliency maps.

Nevertheless, the right-hand side may be much larger than $\tilde{\rho}(x)$, if the inner product between an image and its respective saliency map are almost orthogonal. This is because the Cauchy-Schwarz inequality (see the proofs in the appendix) provides a large upper bound in this case. The inequalities rather serve as an explanation of how the various terms of alignment may deviate from the linearized robustness in the case of a neural network.

\subsection{Alignment and Interpretability}
The above considerations demonstrate how an increase in robustness may induce an increase in the alignment between an input image and its respective saliency map. The initial observation -- which was previously described as an increase in \emph{interpretability} -- may thus be ascribed to this phenomenon. This is especially true in the case of natural images, as exemplified in Figure \ref{fig:dog}. There, what a human observer would deem an increase in interpretability, expresses itself as discriminative portions of the original image reappearing in the saliency map, which naturally implies a stronger alignment. The concepts of alignment and interpretability should however not be conflated completely: In the case of quasi-binary image data like MNIST, 0-regions of the image render the inner product in equation \eqref{eq:multi_alignment_definition} invariant with respect to the saliency map in this region, even if the saliency map e.g. assigns relevance to the absence of a feature in this region. Note however that the saliency map in this region still influences the alignment term through the division by its norm. Additionally, the alignment is also not invariant to the images' representation (color space, shifts, normalization etc.). Still, for most types of image data an increase in alignment in discriminative regions should coincide with an increase in interpretability.

\section{Experiments}\label{sec:experiments}
In order to validate our hypothesis, we trained several models of different adversarial robustness on both MNIST \citep{lecun1990handwritten} and ImageNet \citep{deng2009imagenet} using double backpropagation \citep{drucker1992improving}. For a neural network $f_\theta$ with a softmax output layer, this amounts to minimizing the modified loss
\begin{equation}
    \frac{1}{N} \sum\limits_{i=1}^{N} \left[ \loss(f_\theta(x^{(i)}),y^{(i)}) + \lambda \cdot \| \nabla \loss(f_\theta(x^{(i)}),y^{(i)}) \|^2 \right]
\end{equation}
over the parameters $\theta=(\Theta,b)$. Here, $\lbrace (x^{(i)},y^{(i)})\rbrace_{i=1,\dots,N}$ is the training set and $\loss$ denotes the negative log-likelihood error function. The hyperparameter $\lambda \geq 0$ determines the strength of the regularization. Note that this penalizes the local Lipschitz constant of the loss. As \citep{simon2018adversarial} demonstrate, double backpropagation makes neural networks more resilient to adversarial attacks. By varying $\lambda$, we can easily create models of different adversarial robustness for the same dataset, whose properties we can then compare. \citep{anil2018sorting} previously noted that Lipschitz constrained networks exhibit interpretable saliency maps (without an explanation), which can be regarded as a side-effect of the increase in adversarial robustness.\newline
For the MNIST experiments, we trained each of our 16 models on an NVIDIA 1080Ti GPU with a batch size of 100 for 200 epochs, covering the regularization hyperparameter range from 10 to 180,000, before the models start to degenerate. The used architecture is found in the appendix.\newline
For the experiments on ImageNet, we fine-tuned the pre-trained ResNet50 model from \citep{he2016deep} over 35 epochs on 2 NVIDIA P100 GPUs with a total batch size of 32. We used stochastic gradient descent with a learning rate of 0.0001 and momentum of 0.99. The learning rate was divided by 10 whenever the error stopped improving. For the regularization parameter, we chose $\lambda=10^4,10^{4.5},\dots,10^7$. The experiments were implemented in Tensorflow \citep{tensorflow2015-whitepaper}.

\begin{figure}[t]
    \centering
    \includegraphics[width=.95\linewidth]{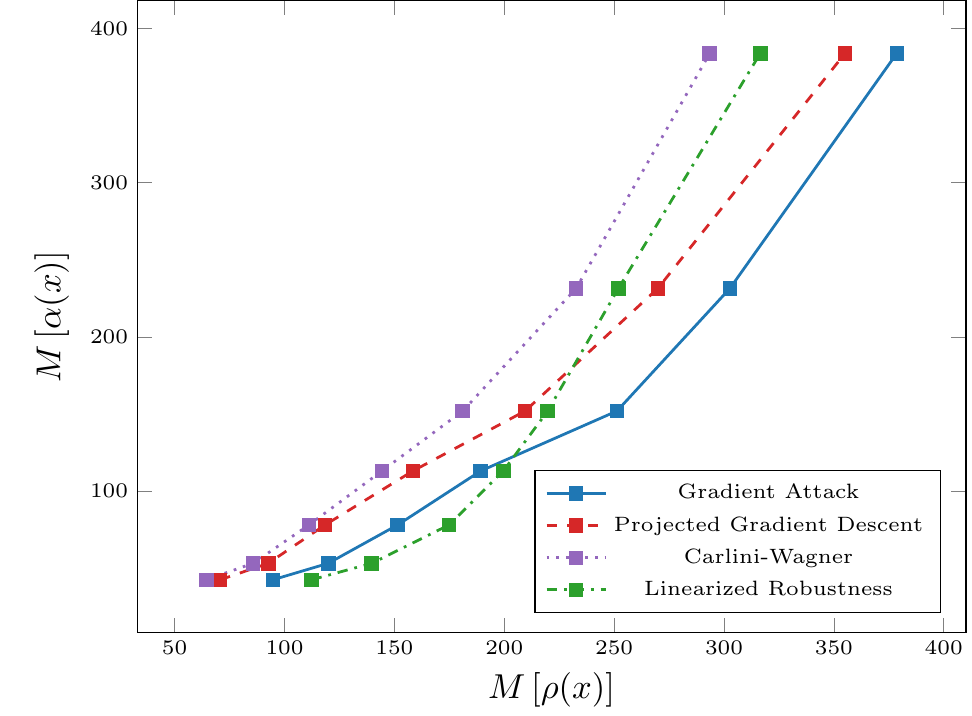}
    \caption{The median alignment increases with the median robustness of the model on ImageNet. Furthermore, the more elaborate attacks consistently find smaller adversarial perturbations than the simple gradient attack. The linearized robustness estimator provides a rather realistic estimation of the algorithmically calculated robustness.}
    \label{fig:alphax_vs_robustness_imagenet}
\end{figure}
\begin{figure}[t]
    \centering
    \includegraphics[width=.95\linewidth]{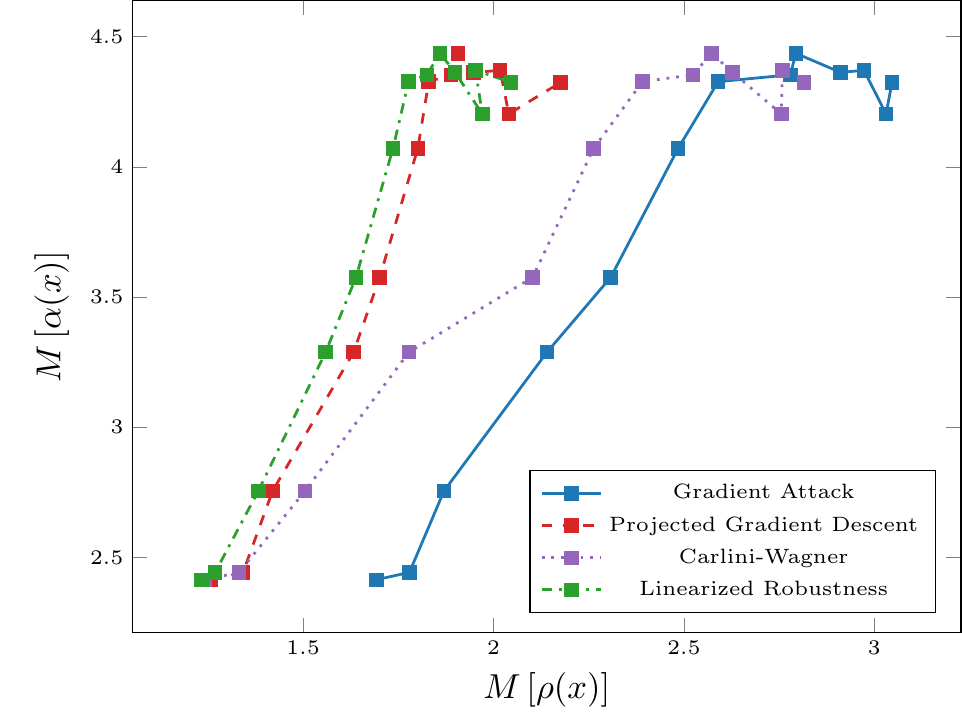}
    \caption{Similar to Figure \ref{fig:alphax_vs_robustness_imagenet}, the median alignment increases with the median robustness of the model on MNIST. Towards the end, some saturation effects are visible.}
    \label{fig:alphax_vs_robustness_mnist}
\end{figure}

\subsection{Robustness and Alignment}
For checking the relation between the alignment and robustness of a neural network, we created 1000 adversarial examples per model on the respective validation set. This was realized using the python library \emph{Foolbox} \citep{rauber2017foolbox}, which offers pre-defined adversarial attacks, three of which we used in this paper: The \texttt{GradientAttack} performs a line search for the closest adversarial example along the direction of the loss gradient. \texttt{L2BasicInterativeAttack} implements the projected gradient descent attack from \citep{kurakin2016adversarial} for the Euclidean metric. Similarly, \texttt{CarliniWagnerL2Attack} (\emph{CW-attack}) is the attack introduced in \citep{carliniwagner} suited for finding the closest adversarial example in Euclidean metric. Additionally, we calculated the linearized robustness $\tilde{\rho}(x)$, which entails calculating $n$ gradients per image for an $n$-class problem.
\\
In Figures \ref{fig:alphax_vs_robustness_imagenet} and \ref{fig:alphax_vs_robustness_mnist}, we investigate how the median alignment depends on  the medians of the different conceptions of robustness. We opted in favor of the median ($M$) instead of the arithmetic mean due to its increased robustness to outliers, which occurred especially when using the gradient attack. In the case of ImageNet (Figure \ref{fig:alphax_vs_robustness_imagenet}), an increase in median alignment with the median robustness is clearly visible for all three estimates of the robustness. On the other hand, the alignment for the MNIST data increases with the robustness as well, but seems to saturate at some point. We will offer an explanation for this phenomenon later.

We now consider the pointwise connection between robustness and alignment. In Figure \ref{fig:scatter_rho_alpha} the two variables are highly-correlated for a model trained on MNIST, pointing towards the fact that the network behaves very similarly to a positive one-homogeneous function. There is however no visible correlation between them on the ImageNet model, which is a consistent behavior throughout the whole experiment cohort. We will later analyse the source of this behavior. The increase in median alignment for ImageNet,  $M\left[\alpha(x)\right]=M\left[ |\langle x,\overline{g}\rangle|\right]$, can still be explained by a statistical argument: If $M\left[ \langle x, \overline{g}\rangle \right]=0$, as approximately true in our ImageNet model, then $M\left[ \alpha(x)\right]$ is the median absolute deviation of $\langle x,\overline{g}\rangle$. In other words, the graph for ImageNet in Figure \ref{fig:scatter_rho_alpha} depicts the dispersion of $\langle x, \overline{g}\rangle$. The above observations also hold well for the binarized alignment.\\
\begin{figure}
    \centering
    \includegraphics[width=.48\linewidth]{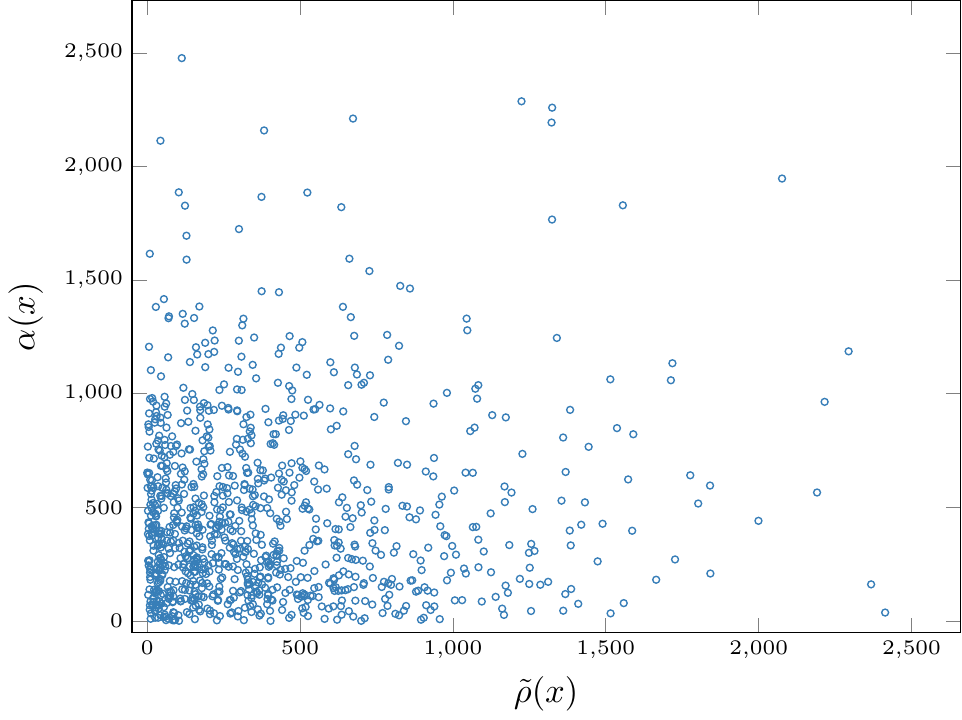}\hspace{10pt}\includegraphics[width=.46\linewidth]{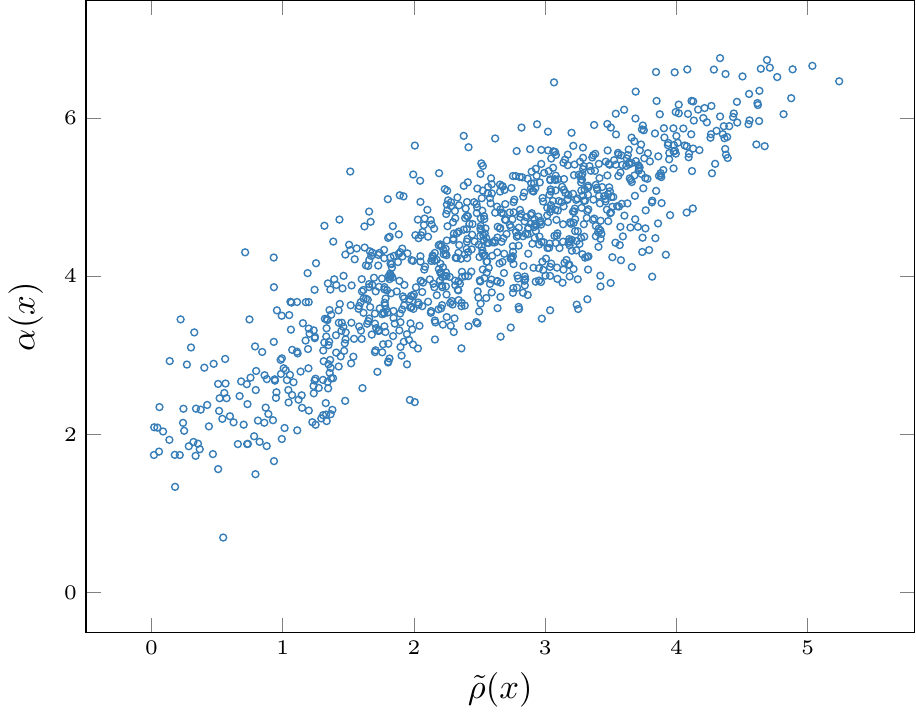}
    \caption{The pointwise relationship between $\tilde{\rho}(x)$ and $\alpha(x)$, exemplified on a model trained on ImageNet (left) and MNIST (right). While the two properties are well-correlated on MNIST (fitting the 'averaged' view from Figure \ref{fig:alphax_vs_robustness_mnist}), there is no visible correlation in the case of ImageNet.} 
    \label{fig:scatter_rho_alpha}
\end{figure}

In Figure \ref{fig:scatter_rho_rho} a tight correlation between $\tilde{\rho}(x)$ and $\rho(x)$ becomes evident. Here, the latter has been calculated using the CW-attack. The linearized robustness model $\tilde{\rho}$ is hence an adequate approximation of the actual robustness $\rho$, even for the highly non-linear neural network models used on ImageNet. Finally note that all used attacks lead to the same general behavior of all quantities investigated (see Figures \ref{fig:alphax_vs_robustness_imagenet} and \ref{fig:alphax_vs_robustness_mnist}).

\begin{figure}
    \centering
    \includegraphics[width=.48\linewidth]{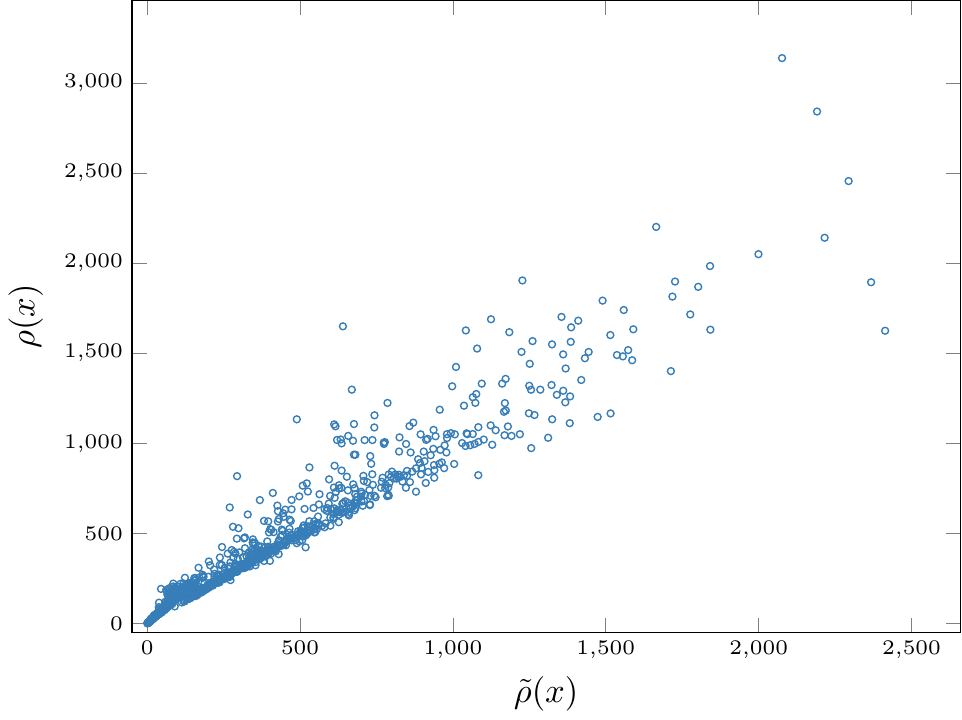}\hspace{10pt}\includegraphics[width=.46\linewidth]{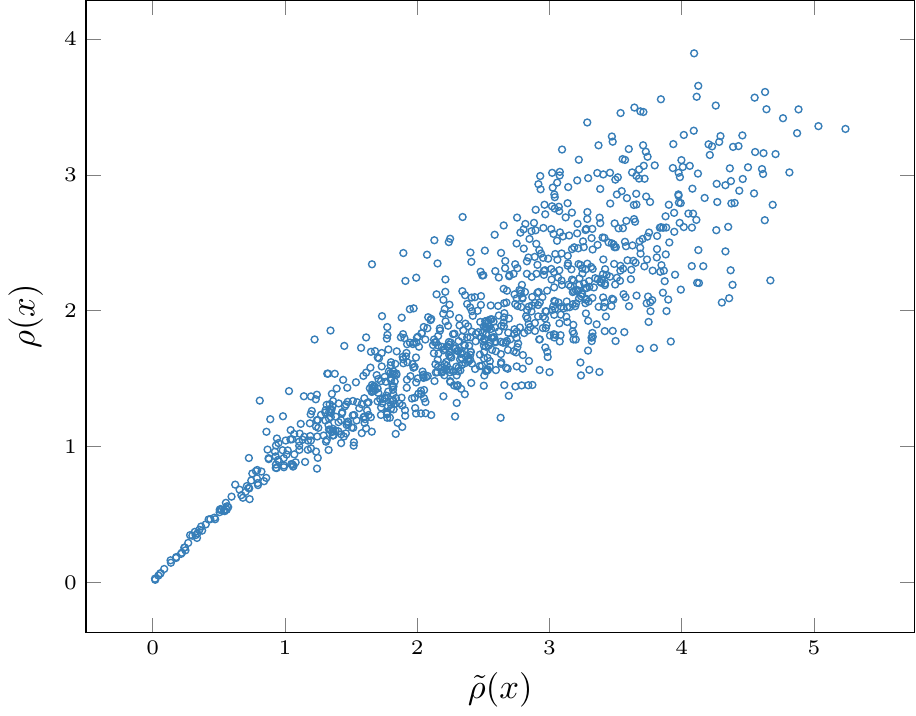}
    \caption{The pointwise relationship between $\tilde{\rho}(x)$ and $\rho(x)$, each calculated for 1000 validation points on a model trained on ImageNet (left) and MNIST (right). $\rho(x)$ was approximately calculated using the CW-attack. In both cases, the correlation is high.} 
    \label{fig:scatter_rho_rho}
\end{figure}
\subsection{Explaining the Observations}
In the last section, we observed some commonalities between the experiments on ImageNet and MNIST, but also some very different behaviors. In particular, two aspects stand out: Why does the median alignment steadily increase for the observed ImageNet experiments, whereas on MNIST this stagnates at some point (Figures \ref{fig:alphax_vs_robustness_imagenet} and \ref{fig:alphax_vs_robustness_mnist})? Furthermore, why are $\tilde{\rho}(x)$ and $\alpha(x)$ so highly-correlated on MNIST but almost uncorrelated on ImageNet (Figure \ref{fig:scatter_rho_alpha})? We turn to Theorems \ref{thm:x_bounds} and \ref{thm:xi_bounds} for answers.\\
Theorem \ref{thm:x_bounds} states that
\begin{equation}\label{eq:2arepeat}
    \tilde{\rho}(x)\leq \alignment(x) + \frac{|\beta^\dagger|}{\|g^\dagger\|},
\end{equation}
where $\beta^\dagger$ is the locally constant term and  $g^\dagger$ is the saliency map of the binarized classifier and $\overline{v}=v/\|v\|$ for $v\neq 0$.
In Figure \ref{fig:theorem2a}, we check how strongly the right-hand side of inequality \eqref{eq:2arepeat} is dominated by $\alignment(x)$, i.e. how large the influence of the locally linear term is in comparison to the locally constant term. For ImageNet, this ratio increases from below 0.55 to almost 0.85, pointing towards a model increasingly governed by its linearized part. On MNIST, this ratio strongly decreases over the robustness's range. Note however that in the weakly regularized MNIST models, the right hand side is extremely dominated by the median alignment in the first place.

A similar analysis can be performed for the second inequality from Theorem \ref{thm:x_bounds}, 
\begin{equation}\label{eq:2brepeat}
    \tilde{\rho}(x)\leq \alpha(x) + \|x\|\cdot\|\overline{g}^\dagger - \overline{g}\| + \frac{|\beta^\dagger|}{\|g^\dagger\|},
\end{equation}
which additionally makes a step from binarized alignment to (conventional) alignment. 

This leads to an additional error term, making the bound significantly less tight than in the previous case. In particular, the proportion of the alignment $\alpha$ on the right-hand side diminishes, confirming our prediction from section \ref{sec:bounds}. Nevertheless, the qualitative behaviors is similar to the previous case, with the $\alpha(x)$ taking up an increasing fraction of the right-hand with increasing robustness.
For MNIST data, the ratio varies little compared to the ratio from the last inequality. This indicates that the remainder term $\|\overline{g}^\dagger - \overline{g}\|$ does not change too strongly over the set of MNIST experiments compared to $\alpha(x)$. We thus deduce that the qualitative relationship between robustness and alignment is fully governed by the error term introduced in \eqref{eq:2arepeat}, i.e. the locally constant term of the logit.

We now do the same for the inequality in Theorem \ref{thm:xi_bounds}, which states that
\begin{equation}
    \begin{aligned}\label{eq:3arepeat}
        \tilde{\rho}(x) &\leq \frac{|\langle \xi, \gamma \rangle|}{ \|\gamma\|} + \| \xi \| \cdot \|\overline{g}^{\dagger} - \overline{\gamma} \|
    \end{aligned}
\end{equation}
for $\xi = x+\tfrac{b^\dagger}{\|g^\dagger\|}\tfrac{g^\dagger}{\|g^\dagger\|}$ and $\gamma=\nabla \Psi^{F(x)}(\xi)$, which gets rid of the additive term $|\beta^\dagger|/\|g^\dagger\|$ from \eqref{eq:2arepeat}. Again, in the case of ImageNet ${|\langle \xi, \overline{\gamma} \rangle|}$ grows more quickly in comparison to $\|\overline{g}^\dagger - \overline{\gamma}\|$, the distance of the normalized gradients, whereas their ratio is approximately constant for MNIST data. \\

\begin{figure}
    \centering
    \includegraphics[width=.46\linewidth]{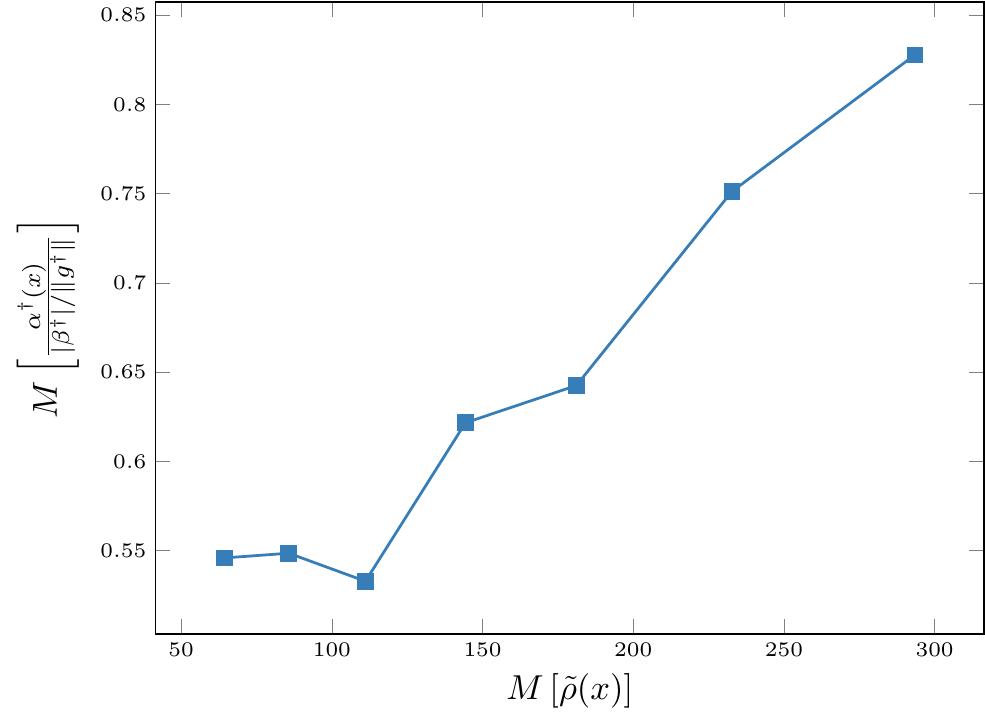}\hspace{10pt}\includegraphics[width=.46\linewidth]{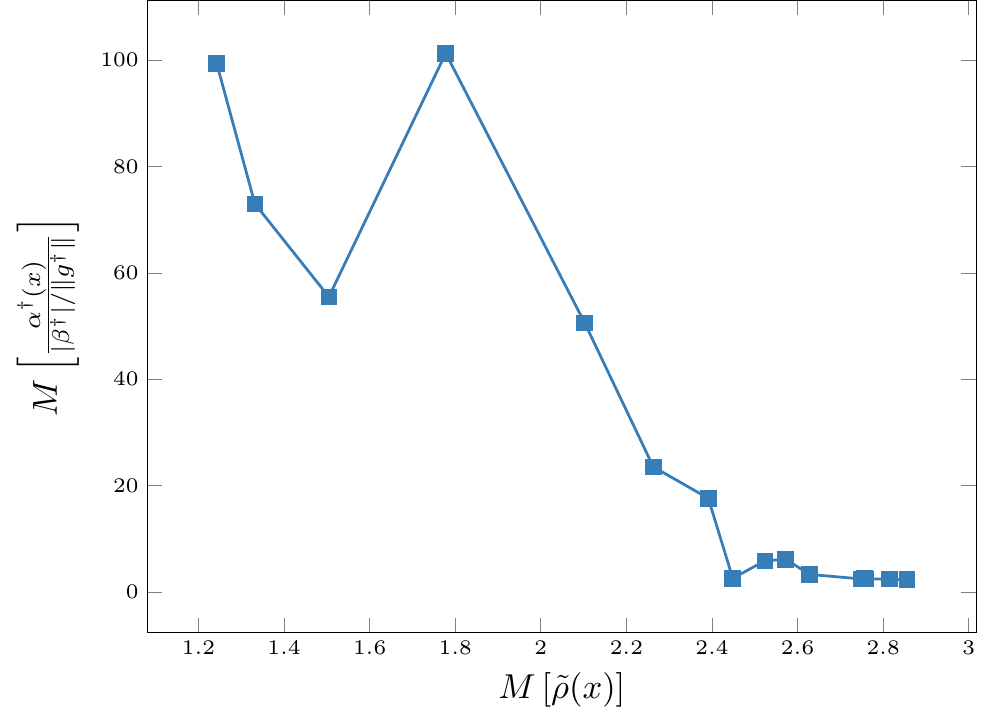}
    \caption{Comparing the size of the summands of inequality \eqref{eq:2arepeat} for the various experiments. In the case of ImageNet (left), $\alpha^\dagger(x)$ takes up an increasing fraction of the right-hand side of the inequality. For MNIST (right), this portion tends to strongly decrease with the robustness. Note however that in this case, $\alignment(x)$ starts out vastly dominating the right-hand side.} 
    \label{fig:theorem2a}
\end{figure}

\begin{figure}
    \centering
    \includegraphics[width=.46\linewidth]{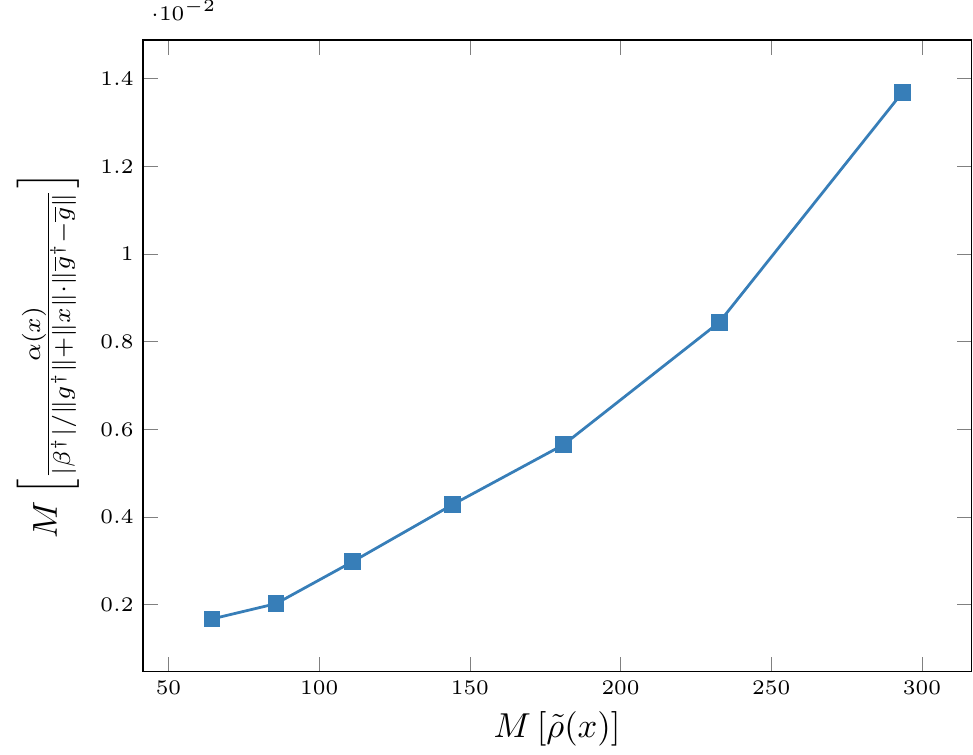}\hspace{10pt}\includegraphics[width=.47\linewidth]{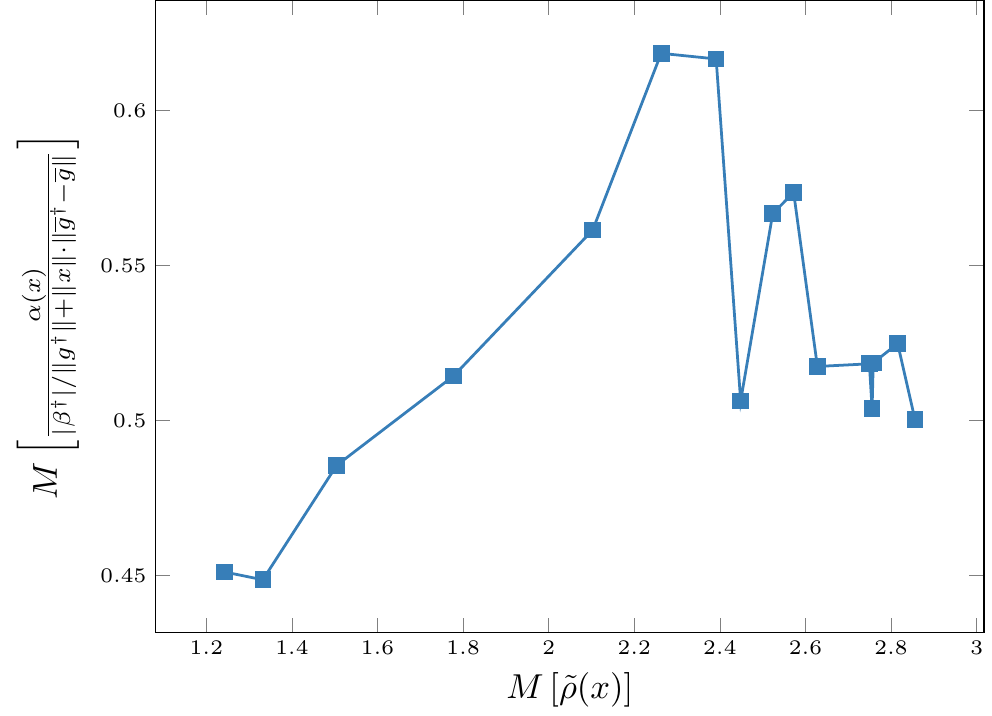}
    \caption{Comparing the size of the summands of inequality \eqref{eq:2brepeat} for the various experiments. For the ImageNet experiments (left), the portion of $\alpha(x)$ of the right-hand side of the inequality increases roughly 7-fold. For MNIST (right), this portion stays roughly constant compared to the variation from Figure \ref{fig:theorem2a}.} 
    \label{fig:theorem2b}
\end{figure}

\begin{figure}
    \centering
    \includegraphics[width=.46\linewidth]{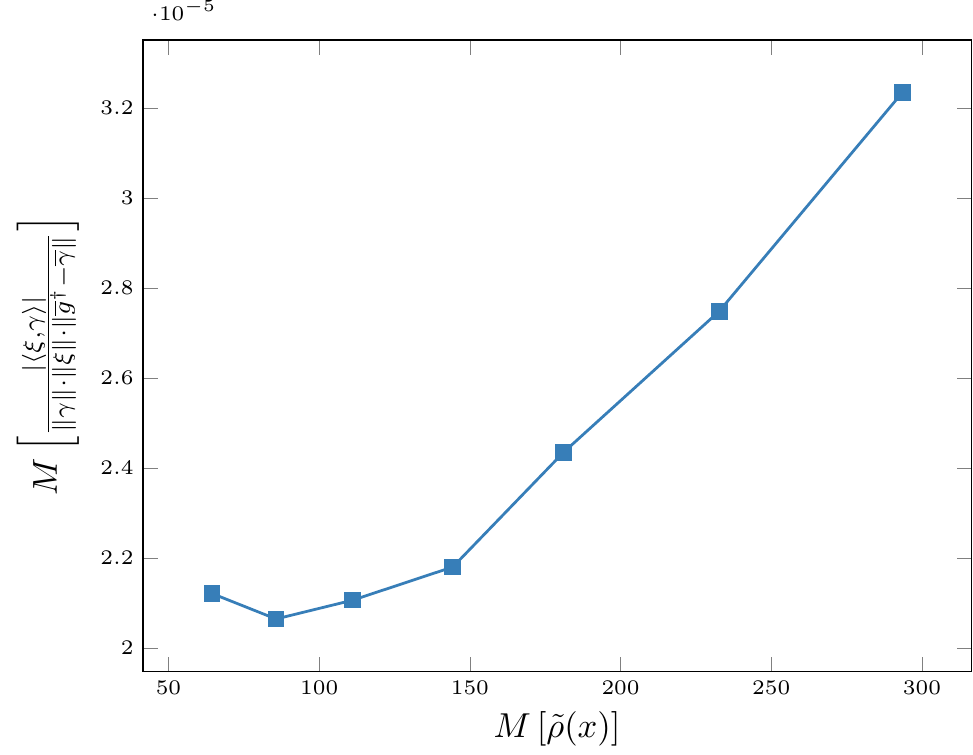}\hspace{10pt}\includegraphics[width=.47\linewidth]{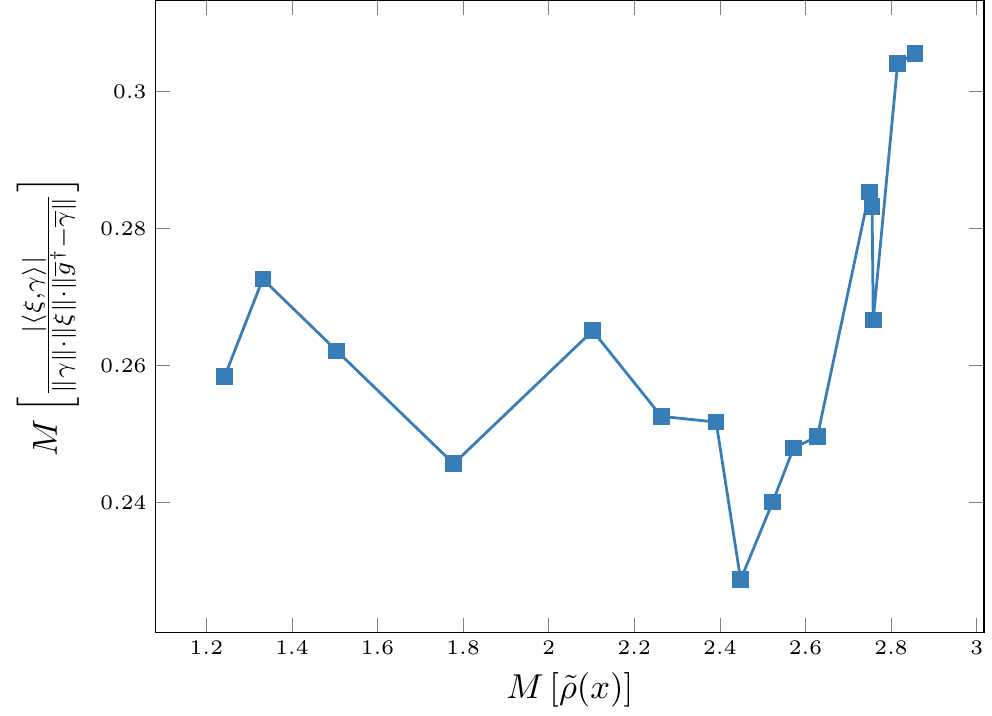}
    \caption{Comparing the size of the summands of inequality \eqref{eq:3arepeat} for the various experiments. In the case of ImageNet (left), the alignment in $\xi$ takes up an increasingly large portion of the right-hand side of the inequality. For MNIST (right), this portion stays roughly constant.} 
    \label{fig:theorem3a}
\end{figure}
To conclude, we have seen that the upper bounds from Theorems \ref{thm:x_bounds} and \ref{thm:xi_bounds} provide valuable information in which ways both the experiments on ImageNet and MNIST are influenced by the respective terms. In the case of ImageNet, we consistently see the alignment terms growing more quickly than the other terms. This might indicate that the growth in alignment stems not only from the growth in the robustness alone, but also from the model becoming increasingly similar to our idealized toy example. In other words, not only does the robustness make the alignment grow, but the connection between these two properties becomes stronger in the case of ImageNet. This is in agreement with the seemingly superlinear growth of the median alignment in Figure \ref{fig:alphax_vs_robustness_imagenet}. 
\\
It is not surprising that a classifier for a problem as complex as ImageNet is highly non-linear, which makes the (pointwise) connection between alignment and robustness rather loose. We hence conjecture that the imposed regularization increasingly restricts the models to be more linear, thereby making them more similar to our initial toy example.\newline
For MNIST, the regularization seems to have the opposite effect: As seen in Figure \ref{fig:theorem2a}, the binarized alignment initially dwarfs the correction term $|\beta^\dagger|/\|g^\dagger\|$ introduced by the locally constant portion of the binarized logit $\Psi^\dagger_x(x)$. As the network becomes more robust, $\Psi^\dagger_x(x)$ is apparently not dominated by the linear terms anymore, while the influence of the locally constant terms (i.e. $\beta^\dagger$) increases. This hypothesis seems sensible, considering MNIST is a very simple problem which we tackled with a comparatively shallow network. This can be expected to yield a model with a low degree of non-linearity. The penalization of the local Lipschitz constant here seems to have the effect of requiring larger locally constant terms $|\beta^\dagger|$, in contrast to the models trained on ImageNet.

We check the validity of these claims by tracking the median size of $|\langle x, g^\dagger \rangle|$ against the median size of $|\Psi^\dagger_x(x)|$ in Figure \ref{fig:linearity}. On MNIST, $M\left[|\langle x, g^\dagger \rangle|\right]$ starts out at approximately $40\%$ of $M\left[|\Psi^\dagger_x(x)|\right]$ and at the end rises to almost $100\%$. Note that this does not indicate that $\beta^i$ is typically close to 0 for all $i$, just that $\beta^\dagger$ is, compared to $\langle x, g^\dagger \rangle$.\newline
On MNIST, this ratio is close to 1 up until $M\left[ \tilde{\rho}(x) \right]\approx2.4$, when it suddenly and quickly falls below $0.5$. This drop is consistent with what we see in Figure \ref{fig:alphax_vs_robustness_mnist}: At around the same point this drop occurs, the alignment starts to saturate. While an increase in the model's median robustness should imply an increase in the model's median alignment, the deviation from linearity weakens the connection between robustness and alignment, such that the two effects roughly cancel out.

In Figure \ref{fig:example_images}, we provide examples for the different gradient concepts we introduced in Theorems \ref{thm:x_bounds} and \ref{thm:xi_bounds}, both for the most robust and non-robust network from our experiment cohort. 

\begin{figure}
    \centering
    \includegraphics[width=.47\linewidth]{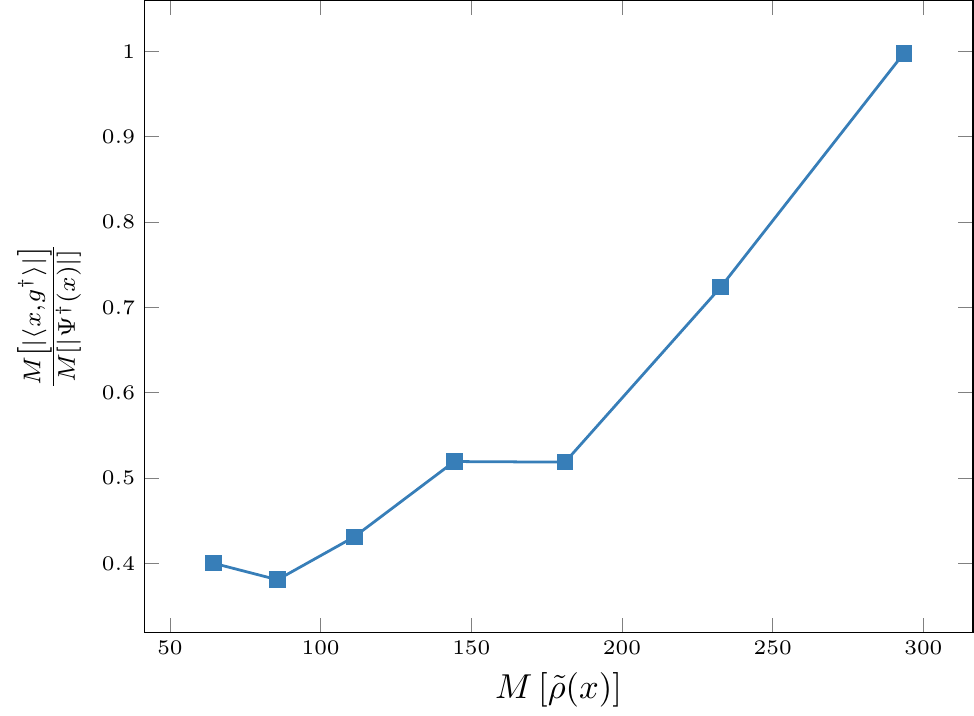}\hspace{10pt}\includegraphics[width=.47\linewidth]{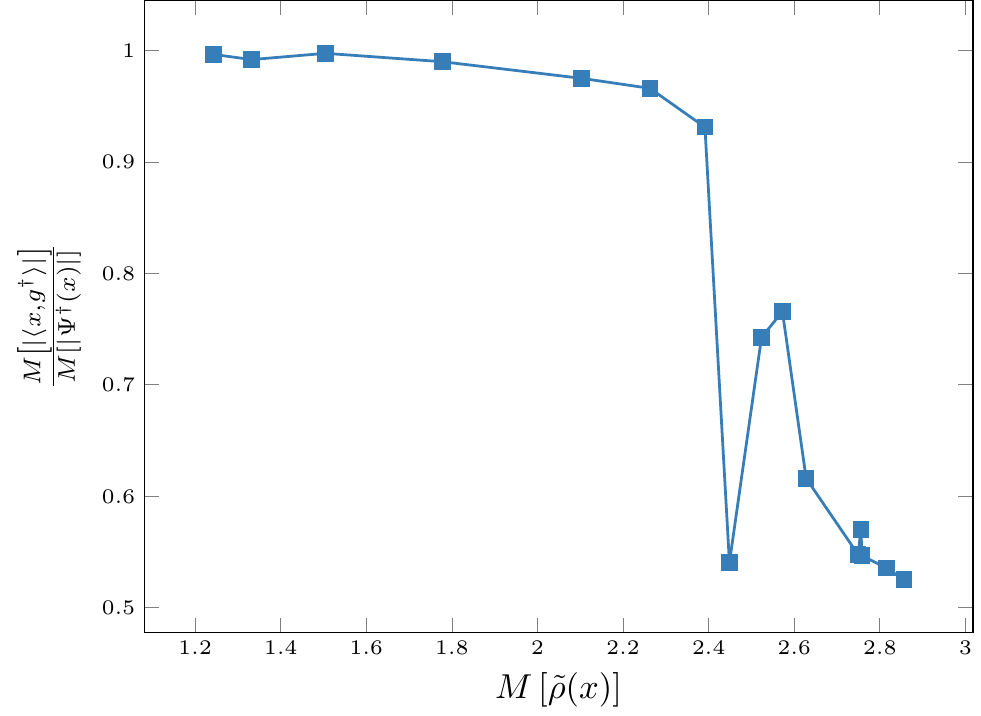}
    \caption{On the ImageNet experiments, the linear term $|\langle x, g^\dagger \rangle|$ takes up an increasing portion of the binarized score $\Psi^\dagger(x)$. In the case of MNIST, $\Psi^\dagger(x)$ is completely dominated by the linear term, before its influence decreases sharply at $M\left[\tilde{\rho}(x) \right]\approx 2.4$.} 
    \label{fig:linearity}
\end{figure}

\section{Conclusion and Outlook}
In this paper, we investigated the connection between a neural network's robustness to adversarial attacks and the interpretability of the resulting saliency maps. Motivated by the binary, linear case, we defined the alignment $\alpha$ as a measure of how much a saliency map matches its respective image. We hypothesized that the perceived increase in interpretability is due to a higher alignment and tested this hypothesis on models trained on MNIST and ImageNet. While on average, the proposed relation holds well, the connection is much less pronounced for individual points, especially on ImageNet. Using some upper bounds for the robustness of a neural network, which we derived using a decomposition theorem, we arrived at the conclusion that the strength of this connection is strongly linked with how similar to a linear model the neural network is locally. As ImageNet is a comparatively complex problem, any sufficiently accurate model is bound to be very non-linear, which explains the difference to MNIST. \\
While this paper shows the general link between robustness and alignment, there are still some open questions. Since we only used one specific robustification method, further experiments should determine the influence of this method. One could explore, whether a different choice of norm leads to different observations. Another future direction of research could be to investigate the degree of (non-)linearity and its connection to this topic. While Theorems \ref{thm:x_bounds} and \ref{thm:xi_bounds} illustrate how the pointwise linearized robustness and alignment may diverge, depending on terms like $g$, $g^\dagger$, $\gamma$ and $\beta^\dagger$, a more in-depth look should focus on \emph{why} and \emph{when} these terms have a certain relationship to each other.\\

From a methodological standpoint, the discovered connection may also serve as an inspiration for new adversarial defenses, where not only the robustness but also the alignment is taken into account. One way of increasing the alignment directly would be through the penalty term $$\lambda \left( \|x\|^2\|\nabla \Psi^i(x)\|^2  - \langle x,\nabla \Psi^i(x) \rangle^2\right),$$ which is bounded from below by 0 via the Cauchy-Schwarz inequality. Any robustifying effects of the increased alignment may however be confounded with the Lipschitz-penalty that the first summand effectively introduces, which necessitates a careful experimental evaluation.

\begin{figure}[H]
    \centering
    \includegraphics[width=.99\linewidth]{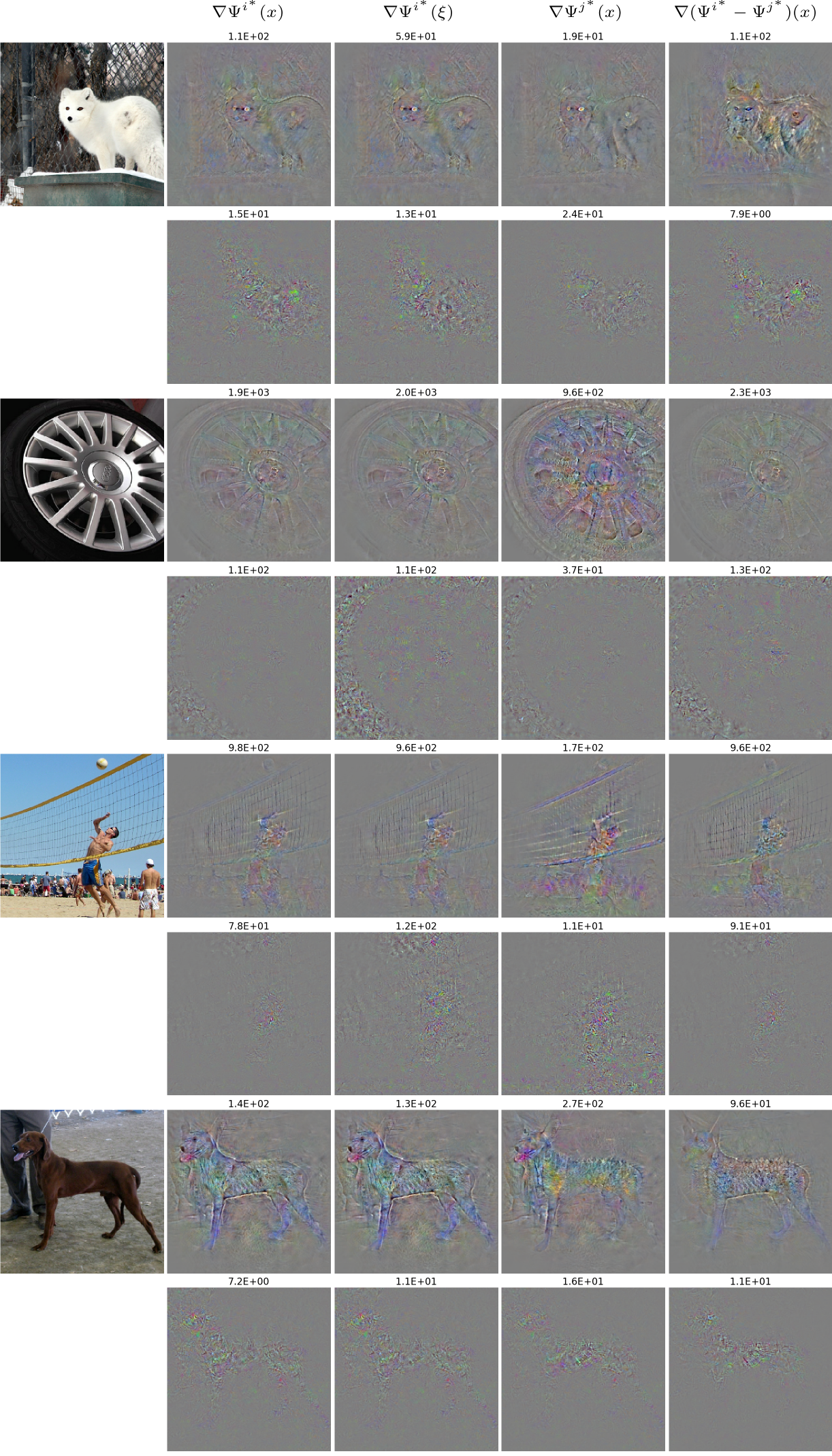}
    \caption{Selected examples from the ImageNet validation set of the different gradients and their respective alignments with $x$, respectively $\xi$. The odd rows are generated with the most robust ImageNet classifier, whereas the even rows are generated by the least robust classifier. The gradient images are individually scaled to fit the color range $[0,255]$.}
    \label{fig:example_images}
\end{figure}

\section*{Acknowledgements}
CE and PM acknowledge funding by the Deutsche Forschungsgemeinschaft (DFG) - Projektnummer 281474342: 'RTG $\pi^3$ - Parameter Identification - Analysis, Algorithms, Applications'. The work by SL was supported by the EPSRC grant EP/L016516/1
for the University of Cambridge Centre for Doctoral Training, the Cambridge Centre for Analysis
and by the Cantab Capital Institute for the Mathematics of Information. CBS acknowledges support from the Leverhulme Trust projects on Breaking the non-convexity barrier and on Unveiling the Invisible, the Philip Leverhulme Prize, the EPSRC grant Nr. EP/M00483X/1, the EPSRC Centre Nr. EP/N014588/1, the European Union Horizon 2020 research and innovation programmes under the Marie Skodowska-Curie grant agreement No 777826 NoMADS and No 691070 CHiPS, the Cantab Capital Institute for the Mathematics of Information and the Alan Turing Institute. We gratefully acknowledge the support of NVIDIA Corporation with the donation of a Quadro P6000 and a Titan Xp GPUs used for this research.

\bibliography{bibliography}
\bibliographystyle{icml2019}

\newpage \phantom{nix}
\newpage
\appendix
\section*{Appendix}

\paragraph{Proof of Equation \eqref{eq:alignment_definition}:}
Note that
\begin{align*}
    &F(x+e) \neq F(x)
    \\
    \Leftrightarrow \ &\langle x+e, z \rangle \langle x, z \rangle  < 0
    \\
    \Leftrightarrow \ & \langle e, z \rangle > |\langle x,z \rangle |.
\end{align*}
The left-hand side is clearly maximized for $e = \|e\| \frac{z}{\|z\|}$, leading to
\begin{align*}
    \|e\| \|z\| > |\langle x,z \rangle |.
\end{align*}
This proves the claim by taking the infimum over $\|e\|$.

\robLemma*
\begin{proof}
As $l(x) \geq \rho(x)$, we can take the infimum in \eqref{eq:generic_robustness} over all perturbations in the local affine component, i.e. $e$ with $\|e\| \leq l(x)$ only. This allows us to reformulate
\begin{align*}
    &F(x+e) \neq F(x)
    \\
    \Leftrightarrow \ &\exists j \neq {i^\ast} : \Psi^j(x+e) > \Psi^{i^\ast}(x+e)
    \\
   \Leftrightarrow \ &\exists j \neq {i^\ast} : \langle \nabla \Psi^j(x) - \nabla \Psi^{i^\ast}(x), e \rangle > \Psi^{i^\ast}(x) - \Psi^j(x).
\end{align*}
The infimum over $\|e\|$ is achieved by choosing $e$ as a multiple of $\nabla \Psi^j(x) - \nabla \Psi^{i^\ast}(x)$. A direct computation then finishes the proof.
\end{proof}

\subsection*{Proofs of Homogenization results}
\begin{lemma}[Euler's Homogeneous Function Theorem]
\label{thm:homo_fcts}
    Let $f:\R^m \to \R$ be a positive one-homogeneous function  that is continuously differentiable on $\R^m \backslash \{ 0 \}$. Then
    \begin{align*}
        f(x) = \langle \nabla f(x), x \rangle
    \end{align*}
\end{lemma}
\begin{proof}
 First note that
 \begin{align*}
     \partial_i f(a x) &= \lim_{t \to 0} \frac{f(ax+te_i)-f(ax)}{t}
     \\
     &= \lim_{t \to 0} \frac{f(ax+ate_i)-f(ax)}{at} = \partial_i f(x).
 \end{align*}
 Hence
 \begin{align*}
     f(x) = \int_0^1 \langle \nabla f(tx), x \rangle \ \text{dt} = \langle \nabla f(x), x \rangle
 \end{align*}
\end{proof}

\homoAlign*
\begin{proof}
Direct consequence of \ref{thm:homo_fcts}.
\end{proof}

\begin{definition}[Neural Networks]
  Define the class of neural networks $\mathcal{N}$ to be any network built on learnable affine transforms (convolutional layers, dense layers) with linear weights $\Theta$ and biases $b$ and ReLU or leaky ReLU activations. The network can include arbitrary skip-connections, batch-normalization layers and max or average pooling layers of arbitrary window size. This in particular includes many state-of-the-art classification networks.
\end{definition}

\begin{lemma}[Homogeneous Networks]
\label{thm:homo_networks}
    For fixed $x$, consider the logit $\Psi^i_{\Theta, b}(x)$ of a network $\Psi_{\Theta, b} \in \mathcal{N}$, where $\Theta$ denotes the linear weights and $b$ the bias vector of the network. Then the function
    \begin{align*}
        f: y \mapsto \Psi^i_{\Theta, b \frac{\|y\|}{\|x\|}}(y),
    \end{align*}
    $f$ is positive one-homogeneous and $f(x) = \Psi^i_{\Theta, b}(x)$.
\end{lemma}
\begin{proof}
Consider first a network consisting of a single layer with linear transform $A$ and bias $b$ with ReLU non-linearity. The associated network function is hence given by $\Psi_{A,b}(x) = (Ax+b)_+$. For this network, we compute for $x$ fixed and any $y$ and $a>0$ as
\begin{align*}
    f(ay) &= \left( A(ay) + b \frac{\|ay\|}{\|x\|} \right)_+ 
    \\
    &= \left( a \cdot Ay + a \cdot b \frac{\|y\|}{\|x\|} \right)_+ = a f(y).
\end{align*}
A single layer is hence positive one-homogeneous. A function consisting of compositions of positive one-homogeneous functions is positive one-homogeneous itself as well, the function $f$ associated to a network consisting of affine transforms and ReLU activations is positive one-homogeneous. All of the operations skip-connections, batch-normalization layers and max or average pooling are positive one-homogeneous as well, thus proving the claim.
\end{proof}

\homoDec*
\begin{proof}
Let $f$ be the functions associated with the network $\Psi_{\Theta, b}^{i}$ as in Lemma \ref{thm:homo_networks}. Then by Lemma \ref{thm:homo_fcts} we can compute the value of $f$ at the point $x$ via
\begin{align*}
    f(x) = \langle x, \nabla_y f(y) |_{y=x}\rangle.
\end{align*}
Note that by construction $f(x) = \Psi^i_{\Theta,b}(x)$. We compute the gradient of $f$ at the point $x$ explicitly as
\begin{align*}
    \nabla_y f(y) |_{y=x} = \nabla_x \Psi^i_{\Theta,b}(x) + \frac{x}{\|x\|^2} \langle b , \nabla_b \Psi^i_{\Theta, b}(x) \rangle.
\end{align*}
Combining these results shows
\begin{align*}
    f(x) &= \langle x, \nabla_x \Psi^i_{\Theta,b}(x) + \frac{x}{\|x\|^2} \langle b , \nabla_b \Psi^i_{\Theta, b}(x) \rangle \rangle\\
    &= \langle x, \nabla_x \Psi^i_{\Theta,b}(x) \rangle + \langle b , \nabla_b \Psi^i_{\Theta, b}(x) \rangle.
\end{align*}
\end{proof}

Recall the notation ${i^\ast}=F(x)$ and $j^\ast$ for the minimizer in $j$ in (\ref*{equ:linearized_robustness}).
\boundOne*
\begin{proof}
We have
\begin{equation*}
    \begin{aligned}
        \tilde{\rho}(x) &= \frac{\Psi^{i^\ast}(x)-\Psi^{j^\ast}(x) }{\|\nabla \Psi^{i^\ast}(x) - \nabla\Psi^{j^\ast}(x)\|}\\
        &= \frac{\langle x, \nabla \Psi^{i^\ast}(x) - \nabla \Psi^{j^\ast}(x) \rangle + \beta^{i^\ast}(x) - \beta^{j^\ast}(x) }{\|\nabla \Psi^{i^\ast}(x) -\nabla\Psi^{j^\ast}(x)\|}\\
        &= \left| \langle x, \overline{g}^\dagger \rangle + \frac{\beta^\dagger}{\|g^\dagger\|} \right| \leq \alpha^\dagger(x) + \frac{|b^\dagger|}{\|g^\dagger\|},
    \end{aligned}
\end{equation*}
using the decomposition theorem and the triangle inequality. Further,
\begin{equation*}
    \begin{aligned}
    &\alpha^\dagger(x) + \frac{|b^\dagger|}{\|g^\dagger\|}\\
    =& \left| \langle x, \overline{g}^\dagger \rangle \right| + \frac{|b^\dagger|}{\|g^\dagger\|} \\
    =& \left| \langle x, \overline{g}^\dagger - \overline{g} + \overline{g} \rangle \right| + \frac{|b^\dagger|}{\|g^\dagger\|} \\
    \leq& \left| \langle x, \overline{g} \rangle \right| + \left| \langle x,\overline{g}^\dagger -  \overline{g} \rangle \right| + \frac{|b^\dagger|}{\|g^\dagger\|} \\
    \leq& \hphantom{.} \alpha(x) + \| x \| \cdot \|\overline{g}^\dagger -  \overline{g} \| + \frac{|b^\dagger|}{\|g^\dagger\|},
    \end{aligned}
\end{equation*}
using the Cauchy-Schwarz inequality.
\end{proof}

\boundTwo*
\begin{proof}
We have
\begin{equation*}
    \begin{aligned}
        \tilde{\rho}(x) &= \frac{ \langle x,g^\dagger \rangle + \beta^\dagger \langle \tfrac{g^\dagger }{\|g^\dagger\|^2} , g^\dagger\rangle}{\|g^\dagger\|}\\
         &= \frac{ \langle x + \tfrac{\beta^\dagger}{\|g^\dagger\|}\tfrac{g^\dagger}{\|g^\dagger\|} , g^\dagger\rangle}{\|g^\dagger\|}\\
         &= \langle \xi,\overline{g}^\dagger \rangle = \langle \xi,\overline{g}^\dagger -\overline{g} + \overline{g} \rangle\\
         &\leq |\langle \xi,\overline{\gamma} \rangle| + \| \xi \| \cdot \|\overline{g}^{\dagger} - \overline{\gamma} \|,
    \end{aligned}
\end{equation*}
using the Cauchy-Schwarz inequality in the same way as in the last theorem.
\end{proof}
\subsection*{MNIST Model Architecture}
Here we describe the architecture that was used for the MNIST models.\newline
\begin{center}
\begin{tabular}{|c|}
\hline Conv2D ($3\times 3$, 'same'), 32 feature maps, ReLU \\
\hline Max Pooling (factor 2) \\
\hline Conv2D ($3\times 3$, 'same'), 64 feature maps, ReLU \\
\hline Max Pooling (factor 2) \\
\hline Conv2D ($3\times 3$, 'same'), 128 feature maps, ReLU \\ 
\hline Max Pooling (factor 2) \\
\hline Dense Layer (128 neurons), ReLU \\
\hline Dropout ($0.5$) \\
\hline Softmax \\\hline 
\end{tabular} 
\end{center}

\end{document}